\def\year{2018}\relax
\documentclass[letterpaper,final]{article} 
\usepackage{long}  
\usepackage{times}  
\usepackage{helvet}  
\usepackage{courier}  
\usepackage{url}  
\usepackage{graphicx}  
\usepackage{xspace}
\usepackage{amsthm,amsmath,amssymb}
\usepackage{color}
\usepackage{ifdraft}
\usepackage{algpseudocode}
\usepackage{algorithm}
\theoremstyle{plain}
\usepackage{subfigure} 

\newtheorem{theorem}{Theorem}
\newtheorem{lemma}[theorem]{Lemma}
\newtheorem{proposition}[theorem]{Proposition}
\newtheorem{corollary}[theorem]{Corollary}
\theoremstyle{definition}
\newtheorem{definition}[theorem]{Definition}
\newtheorem{example}[theorem]{Example}

\ifoptionfinal{
\newcommand{\nb}[1]{}
\newcommand{\cn}{}
}
{
\newcommand{\nb}[1]{\textcolor{red}{$\ddagger$}\marginpar{\scriptsize\raggedright\textcolor{red}{#1}}}
\newcommand{\cn}{\textcolor{blue}{\scriptsize$|$}\marginpar{\scriptsize\raggedright\textcolor{blue}{rpn agree}}}
}

\allowdisplaybreaks

\newcommand{\Imc}{\ensuremath{\mathcal{I}}\xspace}
\newcommand{\Jmc}{\ensuremath{\mathcal{J}}\xspace}
\newcommand{\Tmc}{\ensuremath{\mathcal{T}}\xspace}

\newcommand{\sub}{\ensuremath{\mathsf{sub}}\xspace}

\newcommand{\ALC}{\ensuremath{\mathcal{ALC}}\xspace}
\newcommand{\EL}{\ensuremath{\mathcal{E\!L}}\xspace}

\newcommand{\DL}{{\ensuremath{\mathcal{DL}}}\xspace}

\newcommand{\Inf}{{\ensuremath{\mathsf{Inf}}}\xspace}
\newcommand{\qual}{{\ensuremath{\mathsf{IIC}}}\xspace}

\newcommand{\refine}{\ensuremath{\mathop{\uparrow}}\xspace}
\newcommand{\corefine}{\ensuremath{\mathop{\downarrow}}\xspace}
\DeclareMathOperator*{\argmax}{\mathsf{argmax}}

\begin{document}
%
\title{Repairing Ontologies via Axiom Weakening}
\author{
Nicolas  Troquard, 
Roberto  Confalonieri, 
Pietro  Galliani,
Rafael  Penaloza,  
Daniele  Porello, 
Oliver  Kutz\\
KRDB, Unibz, Bozen-Bolzano, Italy
}

\maketitle
\begin{abstract}
Ontology engineering is a hard and error-prone task, in which small changes may lead to errors,
or even produce an inconsistent ontology. As ontologies grow in size, the need for automated methods
for repairing inconsistencies while preserving as much of the original knowledge as possible increases.
Most previous approaches to this task are based on removing a few axioms from the ontology to regain 
consistency. We propose a new method based on \emph{weakening} these axioms to make them
less restrictive, employing the use of refinement operators. We introduce the theoretical framework for weakening DL ontologies, propose algorithms to repair ontologies based on the framework, and provide an analysis of the computational complexity. Through an empirical analysis made over real-life ontologies, we show that our approach
preserves significantly more of the original knowledge of the ontology than removing axioms.
\end{abstract}

	


\section{Introduction}

Ontology engineering is a hard and error-prone task, where even small changes may lead to unforeseen 
errors, in particular to inconsistency. Ontologies are not only growing in size, they are also increasingly being used in a variety of AI and NLP applications, e.g.,~\cite{Bateman-etal10,PRESTES20131193}.  At the same time, methods to generate ontologies through automated methods gain popularity: e.g., ontology
learning~\cite{LeHi10,SaSB15}, extraction from web resources such as DBpedia~\cite{Auer2007}, or the combination of knowledge from different sources \cite{modularontologies09}.  

Such ontology generation methods are all likely to require ontology repair and refinement steps, and trying to repair an ontology containing hundreds, or even thousands of axioms by hand is infeasible.
For these reasons, it has become fundamental to develop automated methods for repairing ontologies while preserving as much of the original knowledge as possible. 

Most existing ontology repair approaches are based on removing a few axioms to expel
the errors~\cite{ScCo03,kalyanpur2005debugging,kalyanpur2006repairing,BaPS07}.
While these methods are effective, and have been used in practice, they have the
side effect of removing also many potentially wanted implicit consequences. In this paper, we propose
a more fine-grained method for ontology repair based on \emph{weakening} axioms, thus making them
more general. The idea is that, through this weakening, more of the original knowledge is preserved; that
is, our method is less destructive.

We show, both theoretically and empirically, that axiom weakening is a powerful approach for repairing
ontologies. On the theoretical side, we prove that the computational complexity of this task is not 
greater than that of the standard reasoning tasks in description logics. Empirically, we compare the results of
weakening axioms against deleting them, over existing ontologies developed in the life sciences. This comparison shows that our approach
preserves significantly more of the original ontological knowledge than removing axioms, based on an evaluation measure inspecting the preservation of taxonomic structure (see e.g.,\ \cite{Alani:2006:ROA:2127045.2127047,resnik1999a} for related measures).

The main result of this paper is to present a new ontology repair methodology capable of preserving
most of the original knowledge, without incurring any additional costs in terms of computational 
complexity. By thereby preserving more implicit consequences of the ontology, our methodology also provides a 
contribution to the ontology development cycle~\cite{neuhaus2013towards}. Indeed, it can be a useful tool for test-driven
ontology development,  where the preservation of the entailment of competency questions from the weakened ontology can be seen as a measure for the quality of the repair \cite{gruninger1995role,ren2014towards}.




We begin by outlining formal preliminaries, including the introduction of refinement operators and a basic analysis of properties of both, specialisation and generalisation operators. This is followed by a complexity analysis of the problem of computing weakened axioms in our approach. We then present several variations of repair algorithms, a detailed empirical evaluation of their performance, and a quality analysis of the returned ontologies. We close with a discussion of related work and an outlook to future extensions and refinements of the presented ideas. 

\section{Preliminaries}

From a formal point of view, an ontology is a set of formulas in an appropriate logical language
with the purpose of describing a particular domain of interest. 
The precise logic used is in fact not crucial for our approach as most techniques introduced apply to a variety of logics; however, for the sake of clarity we use description
logics (DLs) as well-known examples of ontology languages. We briefly introduce the basic DL
\ALC; for full details see~\cite{BaaderDLH03}. 
The syntax of \ALC is based on two disjoint sets $N_C$ and $N_R$ of \emph{concept names}
and \emph{role names}, respectively.
The set of \emph{\ALC concepts} is generated by the grammar 
\begin{eqnarray*}
C & ::= & A\mid \neg C\mid C\sqcap C\mid C\sqcup C\mid \forall R.C\mid \exists R.C \enspace,
\end{eqnarray*}
where $A\in N_C$ and $R\in N_R$. 
A \emph{TBox} is a finite set of concept inclusions (GCIs) of the form $C\sqsubseteq D$ where $C$ and 
$D$ are concepts. It is used to store terminological knowledge regarding the relationships between 
concepts. 
An \emph{ABox} is a finite set of formulas of the form $C(a)$ and $R(a,b)$, which express 
knowledge about objects in the knowledge domain.

The semantics of \ALC is defined through \emph{interpretations} $I = (\Delta^I, \cdot^I)$, 
where $\Delta^I$ is a non-empty \emph{domain}, and $\cdot^I$ is a function mapping every
individual name to an element of $\Delta^I$, each concept name to a subset of the domain, and each role 
name to a binary relation on the domain.
The interpretation $\mathcal{I}$ is a \emph{model} of the TBox \Tmc if it satisfies all the GCIs in \Tmc. 
Given two concepts $C$ and $D$, we say that $C$ is \emph{subsumed} by $D$ w.r.t.\ the TBox 
$\mathcal{T}$ ($C \sqsubseteq_{\Tmc} D$) if $C^I \subseteq D^I$ for every model 
$I$ of \Tmc. We write $C \equiv_{\Tmc} D$ when $C \sqsubseteq_{\Tmc} D$ and 
$D \sqsubseteq_{\Tmc} C$.
$C$ is \emph{strictly subsumed by} $D$ w.r.t.\ \Tmc ($C \sqsubset_{\Tmc} D$) if 
$C \sqsubseteq_{\Tmc} D$ and $C \not\equiv_{\Tmc} D$.

\EL is the restriction of \ALC allowing only conjunctions and existential 
restrictions~\cite{DBLP:conf/ijcai/BaaderBL05}. It is widely used in biomedical ontologies for describing large terminologies since classification can be computed in polynomial time.\footnote{The OWL 2 EL profile significantly extends the basic \EL logic whilst maintaining its desirable polynomial time complexity, see  \url{https://www.w3.org/TR/owl2-profiles/}.}
%
In the following, \DL denotes either \ALC or \EL, and $\mathcal{L}(\DL, N_C, N_R)$ denotes the set of (complex) concepts that can be built over $N_C$ and $N_R$ in \DL.
\begin{definition}\label{def:sub}
Let \Tmc be a \DL TBox with concept names from $N_{C}$. The set of \emph{subconcepts} of $\Tmc$ is given by 
\begin{equation*}
\sub(\mathcal{T}) = \{\top,\bot\} \cup \bigcup_{C \sqsubseteq D \in \mathcal{T}} \sub(C) \cup \sub(D) \enspace ,
\end{equation*}
where for $C\in N_C\cup\{\top,\bot\}$, $\sub(C)=\{C\}$, and
\begin{eqnarray*}
\sub(\neg C) & = & \{\neg C\} \cup \sub(C) \\
\sub(C \sqcap D) & = & \{C \sqcap D\} \cup \sub(C) \cup \sub(D) \\
\sub(C \sqcup D) & = & \{C \sqcup D\} \cup \sub(C) \cup \sub(D) \\
\sub(\forall R . C) & = & \{\forall R . C\} \cup \sub(C) \\
\sub(\exists R . C) & = & \{\exists R . C\} \cup \sub(C) \enspace.
\end{eqnarray*}
\end{definition}
The size $|C|$ of a concept $C$ is the size of its syntactic tree where for every role $R$, 
$\exists R.$ and $\forall R.$ are individual nodes.
\begin{definition}
The \emph{size} $|C|$ of a concept $C$ is inductively defined as follows. For $C \in N_C \cup \{\top, \bot\}$, $|C| = 1$. Then, $|\lnot C| = 1 + |C|$; $|C \sqcap D| = |C \sqcup D| = 1 + |C| + |D|$; and $|\exists R. C| = |\forall R. C| = 1 + |C|$.
\end{definition}
The \emph{size} $|\mathcal{T}|$ of the TBox $\mathcal{T}$ is 
$\sum_{C \sqsubseteq D \in \mathcal{T}} (|C| + |D|)$.
Clearly, for every $C$ we have $\mathbf{card}(\sub(C)) \leq |C|$ and for every TBox $\mathcal{T}$ we have $\mathbf{card}(\sub(\mathcal{T})) \leq |\mathcal{T}| + 2$.

We now define the upward and downward cover sets of concept names. %
Intuitively,  the upward set of the concept $C$ collects the  most specific subconcepts of the TBox 
$\mathcal{T}$ that subsume $C$; conversely, the  downward set of $C$ collects the most  general 
subconcepts from  $\mathcal{T}$ subsumed by $C$.
The concepts in $\sub(\Tmc)$ are \emph{some} concepts that are relevant in the context of 
$\Tmc$\negmedspace, and that are used as building blocks for generalisations and specialisations. 
%
The properties of $\sub(\mathcal{T})$  
guarantee that the upward and downward cover sets are finite.

\begin{definition}
\label{def:cover}
Let \Tmc be a \DL TBox and $C$ a concept.
The {\em upward cover} and {\em downward cover} of $C$ w.r.t.\ $\Tmc$ are:
\begin{align}
\mathsf{UpCov}_{\Tmc}(C) := {} &\{ D \in \sub(\Tmc)  \mid 
	C \sqsubseteq_{\Tmc} D \text{ and} \nonumber \\ & 
	\nexists. D' \in \sub(\Tmc) \text{ with } C \sqsubset_{\Tmc} D' \sqsubset_{\Tmc} D\} ,\nonumber \\
\mathsf{DownCov}_{\Tmc}(C) := {} &\{ D \in \sub(\Tmc) \mid 
	D \sqsubseteq_{\Tmc} C  \text{ and}  \nonumber\\ &
	\nexists. D' \in \sub(\Tmc) \text{ with } D \sqsubset_{\Tmc} D' \sqsubset_{\Tmc} C \}. \nonumber
\end{align}
\end{definition}
%

Observe that $\mathsf{UpCov}_{\Tmc}$ and $\mathsf{DownCov}_{\Tmc}$ miss interesting refinements. 
Note also that this definition only returns meaningful results when used with a consistent ontology; otherwise it returns the whole set $\sub(\Tmc)$. \nb{Yes i'M right. please check. ... Doesn't it return the all sub(T)? OK: I think you are right because we can never proof that something is a strict inclusion because we can always proof both implications.}\cn 
Hence, when dealing with the repair problem of an inconsistent ontology $O$, we need a derived, consistent `reference ontology' $O^\text{ref}$ to steer the repair process; this is outlined in greater detail in the section on repairing ontologies.
\begin{example}
Let $A,B,C \in N_C$ and
$\Tmc = \{A \sqsubseteq B\}$. We have $\mathsf{UpCov}_{\Tmc}(A \sqcap C) = \{A\}$. Iterating, we get $\mathsf{UpCov}_{\Tmc}(A) = \{A, B\}$ and $\mathsf{UpCov}_{\Tmc}(B) = \{B, \top\}$. 
We could reasonably expect $B \sqcap C$ to be also a generalisation of $A \sqcap C$ w.r.t.\ \Tmc but it will be missed by the iterated application of $\mathsf{UpCov}_{\Tmc}$. Similarly, $\mathsf{UpCov}_{\Tmc}(\exists R. A) = \{\top\}$, while we can expect $\exists R. B$ to be a generalisation of $\exists R. A$. 
\end{example}
%
To take care of these omissions, we introduce a generalisation and specialisation operator.
We denote as $\mathsf{nnf}(C)$ the negation normal form of the concept $C$.
Let $\refine$ and $\corefine$ be two functions from $\mathcal{L}(\DL, N_C, N_R)$ to the powerset of $\mathcal{L}(\DL, N_C, N_R)$.
We define $\zeta_{\refine,\corefine}$, the \emph{abstract refinement operator}, by induction on the structure of concept descriptions as shown in Table~\ref{tab:abstract-refop}.
\begin{table}
\caption{Abstract refinement operator}
\label{tab:abstract-refop}
\resizebox{\columnwidth}{!}{
\parbox{\columnwidth}
{\footnotesize 
\begin{align*}
	\zeta_{\refine,\corefine}(A) = {} & \refine(A) \\
    \zeta_{\refine,\corefine}(\lnot A) = {} & 
			 \{ \mathsf{nnf}(\lnot C) \mid C \in \corefine(A) \} \cup \refine(\lnot A)
    \\
	\zeta_{\refine,\corefine}(\top) = {} &  \refine(\top)\\
	\zeta_{\refine,\corefine}(\bot) = {} & \refine(\bot) \\
  \zeta_{\refine,\corefine}(C \sqcap D) = {} & 
			 \{ C' \sqcap D \mid C' \in \zeta_{\refine,\corefine}(C) \} \cup \\ & \{ C \sqcap D' \mid D' \in \zeta_{\refine,\corefine}(D) \} \cup  \refine(C \sqcap D)
    \\
     \zeta_{\refine,\corefine}(C \sqcup D) = {} & 
			 \{ C' \sqcup D \mid C' \in \zeta_{\refine,\corefine}(C) \} \cup \\ & \{ C \sqcup D' \mid D' \in \zeta_{\refine,\corefine}(D) \} \cup \refine(C \sqcup D)	
    \\
	\zeta_{\refine,\corefine}(\forall R.C) = {} & 
			 \{ \forall R.C' \mid C' \in \zeta_{\refine,\corefine}(C) \} \cup  \refine(\forall R.C)
    \\
\zeta_{\refine,\corefine}(\exists R.C) = {} &
			 \{ \exists R.C' \mid C' \in \zeta_{\refine,\corefine}(C) \} \cup \refine(\exists R.C)
\end{align*}
}}
\end{table}
Complying with 
the previous observation, we define two concrete refinement operators from the abstract operator $\zeta_{\refine,\corefine}$.
\begin{definition}\label{def:refinementoperators}
The \emph{generalisation operator} and \emph{specialisation operator} are defined, respectively, as 
\begin{align*}
\gamma_\Tmc = {} & \zeta_{\mathsf{UpCov}_{\Tmc}, \mathsf{DownCov}_{\Tmc}} \enspace ,\text{and} \\ 
\rho_\Tmc = {} & \zeta_{\mathsf{DowCov}_{\Tmc}, \mathsf{UpCov}_{\Tmc}}\enspace .
\end{align*}
\end{definition}
\noindent Returning to our example, notice that for 
$\Tmc = \{A \sqsubseteq B\}$, we now have $\gamma_\Tmc(A \sqcap C) = \{B \sqcap C, A \sqcap \top, A\}$.


%
%
\begin{definition}
  Given a \DL concept $C$, its \emph{$i$-th refinement iteration} by means of $\zeta_{\refine,\corefine}$ (viz., $\zeta_{\refine,\corefine}^i(C)$) is inductively defined 
  as follows:
  \begin{itemize}
  \item $\zeta_{\refine,\corefine}^0(C) = \{C\}$;
  \item $\zeta_{\refine,\corefine}^{j+1}(C) = \zeta_{\refine,\corefine}^j(C) \cup \bigcup_{C' \in \zeta_{\refine,\corefine}^j(C)} \zeta_{\refine,\corefine}(C')$, \quad $j \geq 0$.
  \end{itemize}
The set of all concepts reachable from $C$ by means of $\zeta_{\refine,\corefine}$ in a finite number of steps is
$\zeta_{\refine,\corefine}^*(C) = \bigcup_{i \geq 0} \zeta_{\refine,\corefine}^i(C).
$
\end{definition}

Some basic properties about $\gamma_\Tmc$ and $\rho_\Tmc$  follow.
\begin{lemma}\label{lem:gen}
For every TBox $\Tmc$:
\newcommand\litem[1]{\item{\bfseries #1:\enspace }}
\begin{enumerate}
\litem{generalisation}\label{item:generalisation} if $X \in \gamma_{\Tmc}(C)$ then $C \sqsubseteq_{\Tmc} X$\\
\textbf{specialisation:\enspace} if $X \in \rho_{\Tmc}(C)$ then $X \sqsubseteq_{\Tmc} C$
\litem{reflexivity}\label{item:reflexivity} if $C \in \sub(\Tmc)$ then $C \in \mathsf{UpCov}_{\Tmc}(C)$ and\linebreak $C \in \mathsf{DownCov}_{\Tmc}(C)$ 
\litem{semantic stability of cover}\label{item:lem-sem-stable-upcover} if $C_1 \equiv_{\Tmc} C_2$ then\linebreak $C_1 \in \mathsf{UpCov}_{\Tmc}(C)$ iff $C_2 \in \mathsf{UpCov}_{\Tmc}(C)$ and\linebreak $C_1 \in \mathsf{DownCov}_{\Tmc}(C)$ iff $C_2 \in \mathsf{DownCov}_{\Tmc}(C)$
\litem{relevant completeness} \label{item:relevant-completeness}$\mathsf{UpCov}_{\Tmc}(C) \subseteq \gamma_{\Tmc}(C)$ and $\mathsf{DownCov}_{\Tmc}(C) \subseteq \rho_{\Tmc}(C)$
\litem{generalisability}\label{item:generalisability} if $C, D \in \sub(\Tmc)$ and $C \sqsubseteq_{\Tmc} D$ then $D \in \gamma_{\Tmc}^*(C)$\\
\textbf{specialisability:} if $C, D \in \sub(\Tmc)$ and $D \sqsubseteq_{\Tmc} C$ then $D \in \rho_{\Tmc}^*(C)$
\litem{trivial generalisability}\label{item:trivial-generalisability} $\top \in \gamma_{\Tmc}^*(C)$\\ 
\textbf{falsehood specialisability:\enspace} $\bot \in \rho_{\Tmc}^*(C)$
\litem{generalisation finiteness}\label{item:generalisation-finiteness} $\gamma_{\Tmc}(C)$ is finite\\
\textbf{specialisation finiteness:\enspace} $\rho_{\Tmc}(C)$ is finite
\end{enumerate}
\end{lemma}
\begin{proof}
See the appendix.
Item~\ref{item:generalisation} is Lemma~\ref{lem:gamma-is-generalisation}.
Item~\ref{item:reflexivity}, item~\ref{item:lem-sem-stable-upcover}, and item~\ref{item:relevant-completeness} are simple consequences of Definition~\ref{def:cover} and Definition~\ref{def:refinementoperators}.
Item~\ref{item:generalisability} is Corollary~\ref{cor:generalisability}.
Item~\ref{item:trivial-generalisability} is in turn a corollary of item~\ref{item:generalisability}.
Item~\ref{item:generalisation-finiteness} is Lemma~\ref{lemma:gamma-finite}.
\end{proof}
%
Although $\gamma_\Tmc(C)$ and $\rho_\Tmc(C)$ are always finite (see Lemma~\ref{lem:gen}.\ref{item:generalisation-finiteness}), this is not the case for $\gamma_\Tmc^*(C)$ and $\rho_\Tmc^*(C)$.
Indeed, their iterated application can produce an infinite chain of refinements. 
\begin{example}
If $\Tmc = \{A {\sqsubseteq} \exists r.A\}$, then $\gamma_\Tmc(A) = \{A, \exists r.A\}$. 
Thus $\gamma_\Tmc(\exists r.A) = \{\exists r.A, \exists r.\exists r.A\} \cup \{\top\}$
(notice that $\top \in \gamma_\Tmc^2(A)$). Continuing the iteration of $\gamma_\Tmc$ on $A$, 
we get $(\exists r.)^k A \in \gamma_\Tmc^k(A)$ for every $k \geq 0$.
\end{example}
This is not a feature caused by the existential quantification alone. Similar examples exist that involve universal quantification, disjunction, and conjunction.\footnote{From the perspective of ontology repair, infinite refinement chains are not an issue since there are always finite chains (Lemma~\ref{lem:gen}.\ref{item:trivial-generalisability}). If needed, it can be simply circumvented by imposing a bound on the size of the considered refinements.} 
Notice that although the covers of two provably equivalent concepts are the same
(Lemma~\ref{lem:gen}.\ref{item:lem-sem-stable-upcover}), it is not the case that $\gamma_\Tmc(C_1) = \gamma_\Tmc(C_2)$ whenever $C_1 \equiv_\Tmc C_2$. For example, with the TBox $\Tmc = \{A \sqsubseteq B\}$, we have $\gamma_\Tmc(A) = \{A, B\}$ and $\gamma_\Tmc(\top \sqcap A) = \{\top \sqcap A, \top \sqcap B, A, B\}$.\nb{possibly hide this}

\section{Complexity}

We now analyse the computational aspects of the refinement operators.
\begin{definition}
Given a TBox \Tmc and concepts $C,D$,
the problems {\sc $\gamma_\Tmc$-membership} and {\sc $\rho_\Tmc$-membership} ask whether $D \in \gamma_\Tmc(C)$ and $D \in \rho_\Tmc(C)$, respectively.
\end{definition}
We show that $\gamma_\Tmc$ and $\rho_\Tmc$ are efficient refinement operators, in the sense that 
deciding {\sc $\gamma_\Tmc$-membership} and {\sc $\rho_\Tmc$\mbox{-}membership} is not harder than 
deciding (atomic) concept subsumption in the underlying logic. Recall that
subsumption  is ExpTime-complete in \ALC and PTime-complete in \EL.
We show that the same complexity bounds hold for
{\sc $\gamma_\Tmc$-membership}.

For proving hardness, we first show that deciding whether $C' \in \mathsf{UpCover}_\Tmc(C)$ is as hard 
as atomic concept subsumption (Theorem~\ref{th:hardness-cover}). Then we show that 
{\sc $\gamma_\Tmc$\mbox{-}membership} is just as hard (Theorem~\ref{th:compl-gamma-lowerbound}).
For the upper bounds, we first establish the complexity of computing the set $\mathsf{UpCover}_\Tmc(C)$ (Theorem~\ref{th:upcov-complexity-membership}). We then show that we can decide 
{\sc $\gamma_\Tmc$\mbox{-}membership} resorting to at most a linear number of computations 
$\mathsf{UpCover}_\Tmc(C')$ (Theorem~\ref{th:compl-gamma-upperbound}).
Combining Theorem~\ref{th:compl-gamma-lowerbound} and Theorem~\ref{th:compl-gamma-upperbound}, we obtain the result.
\begin{theorem}
{\sc $\gamma_\Tmc$-membership} is ExpTime-complete for \ALC and PTime-complete for \EL.
\end{theorem}
Similar arguments can be used to establish the same complexities for {\sc $\rho_\Tmc$-membership}.
\begin{corollary}
{\sc $\rho_\Tmc$-membership} is ExpTime-complete for \ALC and PTime-complete for \EL.
\end{corollary}
The remainder of this section provides the details.
We first prove a technical lemma used in the 
reduction from concept subsumption to deciding whether $C' \in \mathsf{UpCover}_\Tmc(C)$.
\begin{lemma}
\label{lem:bot}
Let \Tmc be a \DL TBox and $X\notin\sub(\Tmc)$. 
Then, for every model $\Imc$ of 
$\Tmc':=\Tmc\cup\{X\sqcap B\sqsubseteq \top\}$
there is a model $\Jmc$ of $\Tmc'$ such that
\begin{enumerate}
\item $X^\Jmc=\emptyset$, and
\item for every $C\in\sub(\Tmc)$, $C^\Imc=C^\Jmc$.
\end{enumerate}
\end{lemma}
\begin{proof}
We define the interpretation \Jmc where all role
names are interpreted as in \Imc, and for every
concept name $A\in N_C$ 
\[
A^\Jmc := 
	\begin{cases}
	 \emptyset & \text{if $A=X$} \\
     A^\Imc & \text{otherwise}.
	\end{cases}
\]
Since $X$ only appears in a tautology, \Jmc is also
a model of $\Tmc'$. Using induction on the structure
of the concepts, it is easy to show that the second
condition of the lemma holds.
\end{proof}

The following theorem is instrumental in the proof of Theorem~\ref{th:compl-gamma-lowerbound}.
\begin{theorem}\label{th:hardness-cover}
Let \Tmc be a \DL TBox and let $C$ be an arbitrary \DL concept.
Deciding whether $D\in\mathsf{UpCov}_{\Tmc}(C)$
is as hard as deciding atomic subsumption w.r.t.\ a TBox over \DL.
\end{theorem}
\begin{proof}
We propose a reduction from the problem of deciding atomic subsumption w.r.t.\ a TBox.
Let \Tmc be a \DL TBox, and $A,B$ be two concept
names. We assume w.l.o.g.\ that $\{A,B\}\subseteq\sub(\Tmc)$. Define the new TBox
\[
\Tmc' := \Tmc \cup \{X\sqcap B\sqsubseteq \top\} \enspace ,
\]
where $X$ is a new concept name (not appearing in
\Tmc).%
\footnote{We use this tautology only
to ensure that $X\sqcap B\in\sub(\Tmc')$, to satisfy
the restriction on the definition of the upward
cover. 
}
We show that $A\sqsubseteq_\Tmc B$ iff
$X\sqcap B\in\mathsf{UpCov}_{\Tmc'}(X\sqcap A)$.

\noindent[$\Rightarrow$]
If $A\sqsubseteq_\Tmc B$, then 
$X\sqcap A\sqsubseteq_\Tmc X\sqcap B$, and hence it
also holds $X\sqcap A\sqsubseteq_{\Tmc'} X\sqcap B$.
Assume that there is some $E\in\sub(\Tmc')$
with $X\sqcap A\sqsubset_{\Tmc'} E\sqsubset_{\Tmc'} X\sqcap B$.
Then $E$ cannot be $X\sqcap B$, nor $X$. Hence $E\in\sub(\Tmc)$. Let \Imc be an arbitrary model of $\Tmc'$. By Lemma~\ref{lem:bot}, there is a model $\Jmc$ with
$E^\Imc=E^\Jmc$ and $X^\Jmc = \emptyset$. But since by assumption $E \sqsubseteq_{\Tmc'} X \sqcap B$, it must be that $E^\Jmc = \emptyset$, and hence $E^\Imc = \emptyset$. It then follows
that for every model \Imc of $\Tmc'$, we have $E^\Imc=\emptyset$,
which is a contradiction with the assumption $X\sqcap A\sqsubset_{\Tmc'} E$. 
We conclude that $X\sqcap B\in\mathsf{UpCov}_{\Tmc'}(X\sqcap A)$.

\noindent[$\Leftarrow$]
If $A\not\sqsubseteq_\Tmc B$, there is a 
model \Imc of \Tmc with $A^\Imc\not\subseteq B^\Imc$.
We can extend this interpretation to a model $\Jmc$ 
of $\Tmc'$ by setting $X^\Jmc=\Delta^\Imc$, and
$A^\Jmc=A^\Imc$ for all other concept names; and
$r^\Jmc=r^\Imc$ for all role names.
Then $(X\sqcap A)^\Jmc\not\subseteq(X\sqcap B)^\Jmc$,
and hence $X\sqcap B\notin \mathsf{UpCov}_{\Tmc'}(X\sqcap A)$.
\end{proof}

\begin{theorem}
Let \Tmc be a \DL TBox and let $C$ be an arbitrary \DL concept. Deciding whether $D \in \mathsf{UpCov}_{\Tmc}(C)$ can be done in exponential time when $\DL = \ALC$ and in polynomial time when $\DL = \EL$.
\end{theorem}
\begin{proof}
An algorithm goes as follows. If $D \not \in \sub(\Tmc)$ or $C\not\sqsubseteq_\Tmc D$, return {\sf false}. 
Then, for every $E \in \sub(\Tmc)$, check whether: (1)~$C \sqsubseteq_\Tmc E$, 
(2)~$E \sqsubseteq_\Tmc D$, (3)~$E \not\sqsubseteq_\Tmc C$, and (4)~$D \not\sqsubseteq_\Tmc E$. 
If conditions~(1)--(4) are all satisfied, return {\sf false}.
Return {\sf true} after trying all $E \in \sub(\Tmc)$. The routine requires at most $1 + 4\times \mathbf{card}(\sub(\Tmc))$ calls 
to the subroutine for \DL concept subsumption.
Since $\mathbf{card}(\sub(\Tmc))$ is linear in $|\Tmc|$, the overall routine runs in exponential time when $\DL = \ALC$ 
and in polynomial time when $\DL = \EL$.
\end{proof}

The following theorem is instrumental in the proof of Theorem~\ref{th:compl-gamma-upperbound}.
\begin{theorem}\label{th:upcov-complexity-membership}
Let \Tmc be a \DL TBox and let $C$ be a \DL concept. $\mathsf{UpCov}_{\Tmc}(C)$ is computable in 
exponential time when $\DL = \ALC$ and in polynomial time when $\DL = \EL$.
\end{theorem}
\begin{proof}
It suffices to check for every $D \in \sub(\Tmc)$ whether $D \in \mathsf{UpCov}_\Tmc(C)$ and collect 
those concepts for which the answer is positive. Since $\mathbf{card}(\sub(\Tmc))$ is linear in the size of \Tmc, the result holds. 
\end{proof}
%
\begin{lemma}
\label{lem:technical2}
Let \Tmc be a \DL TBox, $C$ a \DL concept, and $X\notin\sub(\Tmc)$.
Define $\Tmc' := \Tmc \cup \{X \equiv C \}$. 
If $D \in \sub(\Tmc)$ then $D \in \mathsf{UpCov}_{\Tmc}(C)$ iff $D \in \mathsf{UpCov}_{\Tmc'}(C)$.
\end{lemma}
\begin{proof}
We have $\sub(\Tmc') {=} \sub(\Tmc) \cup \{X\}$. Let $D \in \sub(\Tmc)$.
Suppose $D \in \mathsf{UpCov}_{\Tmc'}(C)$. Then $C \sqsubseteq_{\Tmc'} D$ and there is no 
$E \in \sub(\Tmc')$ such that $C \sqsubset_{\Tmc'} E \sqsubset_{\Tmc'} D$. We thus have 
$C \sqsubseteq_{\Tmc} D$.
Since $\sub(\Tmc) \subset \sub(\Tmc')$ there is no $E \in \sub(\Tmc)$ such that 
$C \sqsubset_{\Tmc} E \sqsubset_{\Tmc} D$.

Let $D \in \mathsf{UpCov}_{\Tmc}(C)$. Then $C \sqsubseteq_\Tmc D$ and 
$C \sqsubseteq_{\Tmc'} D$.
Moreover, there is no $E \in \sub(\Tmc)$ with $C \sqsubset_{\Tmc} E \sqsubset_{\Tmc} D$. So there is no $E \in \sub(\Tmc)$ such that $C \sqsubset_{\Tmc'} E \sqsubset_{\Tmc'} D$.

Since $X \equiv_{\Tmc'} C$, it is not the case that $C \sqsubset_{\Tmc'} X \sqsubset_{\Tmc'} D$.
Since $\sub(\Tmc') = \sub(\Tmc) \cup \{X\}$, there is no $E \in \sub(\Tmc')$ such that $C \sqsubset_{\Tmc'} E \sqsubset_{\Tmc'} D$. Then $D \in \mathsf{UpCov}_{\Tmc'}(C)$.
\end{proof}

\begin{theorem}\label{th:compl-gamma-lowerbound}
Deciding {\sc $\gamma_{\Tmc}$-membership} is as hard as deciding whether $D\in\mathsf{UpCov}_{\Tmc}(C)$.
\end{theorem}
\begin{proof}
Let \Tmc be a \DL TBox, $C$ a concept, and $X\notin\sub(\Tmc)$.
Define $\Tmc' := \Tmc \cup \{X \equiv C \}$.
For every concept $D \not = X$, we show that $D \in \mathsf{UpCov}_{\Tmc}(C)$ iff 
$D \in \gamma_{\Tmc'}(X)$.

By Lemma~\ref{lem:technical2}, $D \in \mathsf{UpCov}_{\Tmc}(C)$ iff $D \in \mathsf{UpCov}_{\Tmc'}(C)$.
Since $X \equiv_{\Tmc'} C$, Lemma~\ref{lem:gen}.\ref{item:lem-sem-stable-upcover} yields 
$D \in \mathsf{UpCov}_{\Tmc'}(C)$ iff $D \in \mathsf{UpCov}_{\Tmc'}(X)$.
As $X$ is a concept name, by definition of $\gamma$ we have 
$D \in \mathsf{UpCov}_{\Tmc'}(X)$ iff $D \in \gamma_{\Tmc'}(X)$.
\end{proof}

\begin{theorem}\label{th:compl-gamma-upperbound}
Let \Tmc be a \DL TBox and $C$ a concept. {\sc $\gamma_{\Tmc}$-membership} can be decided 
in exponential time when $\DL = \ALC$ and in polynomial time when $\DL = \EL$.
\end{theorem}
\begin{proof}
We can decide whether $\gamma_{\Tmc}(C)$ contains a particular concept by computing only a linear 
number of times $\mathsf{UpCov}_{\Tmc}(C')$, where $|C'|$ is linearly bounded by $|C'| + |\Tmc|$.
Theorem~\ref{th:upcov-complexity-membership} tells us that each of these computations can be done in 
exponential time when $\DL = \ALC$ and in polynomial time when $\DL = \EL$. This yields an exponential 
time procedure when $\DL = \ALC$ and a polynomial time procedure when $\DL = \EL$.
\end{proof}

\section{Repairing Ontologies}
%
Our refinement operators can be used as components of a method for repairing inconsistent ontologies by weakening, instead of removing, problematic axioms.

Given an inconsistent ontology $O$, we proceed as described in Algorithm \ref{algo:repair_weaken}. 
Briefly, we first need to find a consistent subontology $O^\text{ref}$ of $O$ to serve as \emph{reference ontology} in order to be able to compute a non-trivial upcover and downcover. 
The \emph{brave} approach (which we use in our evaluation) picks a random maximally consistent subset of $O$ and chooses it as reference ontology $O^\text{ref}$. The \emph{cautious} approach takes as $O^\text{ref}$ the intersection of all maximally consistent subsets~\cite{LuPe14,LLRRS10}. 
While the brave approach is faster to compute and still guarantees to find solutions, the cautious approach 
has the advantage of not excluding certain repairs a priori. However, it also returns, e.g., a much impoverished upcover.

Once a reference ontology $O^\text{ref}$ has been chosen, and as long as $O$ is inconsistent, we select a ``bad axiom'' and replace it with a random weakening of it with respect to $O^\text{ref}$. 
In view of evaluation, we consider two variants of the subprocedure $\text{FindBadAxiom}(O)$. The first variant (`mis') randomly samples a number of minimally inconsistent subsets $I_1, I_2, \ldots I_k \subseteq O$ and returns one axiom from the ones occurring the most often, i.e., an axiom from the set 
$\argmax_{\phi \in O} (\mathbf{card}(\{j \mid \phi \in I_j \text{ and } 1 \leq j \leq k\}))$. 
The second variant (`rand') of $\text{FindBadAxiom}(O)$ merely returns an axiom in $O$ at random.
%
%

The set of all weakenings of an axiom with respect to a reference ontology is defined as follows:
\begin{definition}[Axiom weakening] Given a subsumption axiom $C \sqsubseteq D$ of $O$, the set of (least) \emph{weakenings} of $C \sqsubseteq D$ w.r.t.\ $O$, denoted by $g_{O}(C \sqsubseteq D)$ is the
set of all axioms $C'\sqsubseteq D'$ such that
\[
C' \in \rho_{O}(C) \text{ and } D' \in \gamma_{O}(D) \enspace .
\]
Given an assertional axiom $C(a)$ of $O$, the set of (least) \emph{weakenings} of $C(a)$, denoted $g_{O}(C(a))$ is the set of all axioms $C'(a)$ such that
\[
C' \in \gamma_O(C) \enspace .
\]
\end{definition}
The subprocedure $\text{WeakenAxiom}(\phi, O^\text{ref})$ randomly returns one axiom in $g_O(\phi)$.
For every subsumption or assertional axiom $\phi$, the axioms in the set $g_{O}(\phi)$ are indeed weaker than $\phi$.

\begin{lemma} For every subsumption or assertional axiom $\phi$, if $\phi' \in g_{O}(\phi)$, then $\phi \models_{O} \phi'$.
\end{lemma}

\begin{proof}
Suppose $C' \sqsubseteq D' \in g_{O}(C \sqsubseteq D)$. Then, by definition of $g_{O}$ and Lemma~\ref{lem:gen}.\ref{item:generalisation}, $C'\sqsubseteq C$ and $D \sqsubseteq D'$ are inferred from $O$. Thus, by transitivity of subsumption, we obtain that $C \sqsubseteq D \models_{O} C' \sqsubseteq D'$. For the weakening of assertions, the result follows immediately from Lemma~\ref{lem:gen}.\ref{item:generalisation} again.
\end{proof}

\begin{algorithm}[t]
\begin{algorithmic}
\State $O^\text{ref} \gets \text{MaximallyConsistent($O$)}$
\While{$O$ is inconsistent}
\State BadAx $\gets$ FindBadAxiom($O$)
\State WeakerAx $\gets$ WeakenAxiom(BadAx, $O^\text{ref}$)
\State $O \gets O \setminus \{\text{BadAx}\} \cup \{\text{WeakerAx}\}$
\EndWhile
\State Return $O$
\end{algorithmic}
\caption{RepairOntologyWeaken($O$)}
	\label{algo:repair_weaken}
\end{algorithm}

Clearly, substituting an axiom $\phi$ with one axiom from $g_O(\phi)$ cannot diminish the set of interpretations of an ontology.
By Lemma~\ref{lem:gen}.\ref{item:trivial-generalisability}, any subsumption axiom is a finite number of refinement steps away from the trivial axiom $\bot \sqsubseteq \top$. Any assertional axiom $C(a)$ is also a finite number of generalisations away from the trivial assertion $\top(a)$.
It follows that by repeatedly replacing an axiom with one of its weakenings, the weakening procedure will eventually obtain an ontology with some interpretations. 
Hence, the algorithm will eventually terminate.

In the next section, we compare Algorithm \ref{algo:repair_weaken} with Algorithm \ref{algo:repair_remove}, which merely removes bad axioms until an ontology becomes consistent. We do so for both variants `mis' and `rand' of $\text{FindBadAxiom}(O)$.
As we will see, Algorithm~\ref{algo:repair_weaken} generally allows us to obtain consistent ontologies which retain significantly more of the informational content of the axioms of the original (and inconsistent) ontology than the ones obtained through Algorithm~\ref{algo:repair_remove}. This is most significant with the `mis' variant of $\text{FindBadAxiom}(O)$ which reliably pinpoints the problematic axioms.


\begin{algorithm}[t]
\begin{algorithmic}
\While{$O$ is inconsistent}
\State BadAx $\gets$ FindBadAxiom($O$)
\State $O \gets O \setminus \{\text{BadAx}\}$
\EndWhile
\State Return $O$
\end{algorithmic}
\caption{RepairOntologyRemove($O$)}
	\label{algo:repair_remove}
\end{algorithm}
\begin{table}[t]
	\caption{BioPortal ontologies considered for experimental validation}
	\label{table:ontologies}
\resizebox{\hsize}{!}{
	\begin{tabular}{l l}
		Abbreviation & Name\\ 
		\hline
		bctt & Behaviour Change Technique Taxonomy\\
		co-wheat & Wheat Trait Ontology \\ 
		elig & Eligibility Feature Hierarchy\\
		hom & Homology and Related Concepts in Biology\\
		icd11 & Body System Terms from ICD11\\
		ofsmr & Open Food Safety Model Repository\\
		ogr & Ontology of Geographical Region\\
		pe & Pulmonary Embolism Ontology\\
		taxrank & Taxonomic Rank Vocabulary\\
		xeo & XEML Environment Ontology
	\end{tabular}
    }
\end{table}
\begin{figure}[t]
{\includegraphics[width=1\columnwidth]{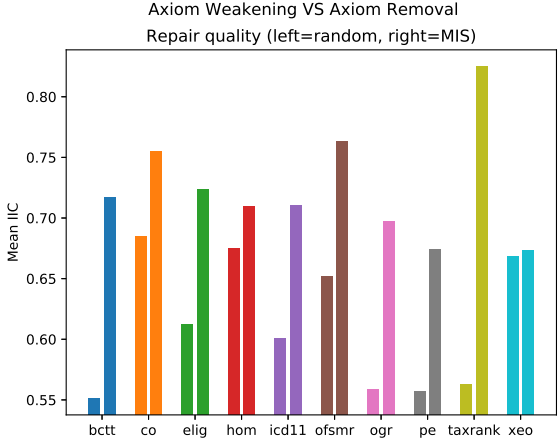}}
\caption{Comparing weakening-based ontology repair with removal-based ontology repair. 
Mean \qual of weakening-based against removal-based repair for each ontology, when choosing axioms at random (left) or by sampling minimally inconsistent sets (right).}
\label{fig:evaluation}
\end{figure}
\begin{table}
\begin{center}
	\begin{tabular}{l | c c }
		& Random & MIS\\
		\hline
		bctt & 0.55 (0.35) & \textbf{0.72 (0.36)}\\
		co-wheat & \textbf{0.69 (0.29)} & \textbf{0.76 (0.31)}\\
		elig & \textbf{0.61 (0.30)} & \textbf{0.72 (0.27)}\\
		hom & \textbf{0.68 (0.26)} & \textbf{0.71 (0.31)}\\
		icd11 & 0.60 (0.30) & \textbf{0.71 (0.40)}\\
		ofsmr & \textbf{0.65 (0.31)} & \textbf{0.76 (0.29)}\\
		ogr & 0.56 (0.32) & \textbf{0.70 (0.35)}\\
		pe & 0.56 (0.33) & \textbf{0.67 (0.41)}\\
		taxrank & 0.56 (0.31) & \textbf{0.82 (0.36)}\\
		xeo & \textbf{0.67 (0.29)} & \textbf{0.67 (0.34)}
	\end{tabular}
    \end{center}
    \caption{Mean and standard deviation (in parentheses) of \qual between RepairOntologyWeaken and RepairOntologyRemove, both when choosing axioms at random (left column) and by sampling minimally inconsistent sets (right). Bolded values are significant $(p < 0.05)$ with respect to both Wilcoxon and T-test with Holm-Bonferroni correction; non-bolded values were not significant for either.}\label{table:results}
\end{table}



\section{Evaluation}
The question of which one of two consistent repairs $O_1$ and $O_2$ of a given inconsistent ontology $O$ is preferable is not, in general, well-defined. In this work, we compare two such repairs by taking into account the corresponding \emph{inferred class hierarchies}. To this end, we define:
\[
	\Inf(O_i) = \{A \sqsubseteq B : A, B \in N_C, O_i \models A \sqsubseteq B\} \enspace.
\]
The intuition behind the choice of measure is that if $\mathbf{card}(\Inf(O_1) \setminus \Inf(O_2))>\mathbf{card}(\Inf(O_2) \setminus \Inf(O_1))$ (that is, if there exist more subsumptions between classes which can be inferred in $O_1$ but not in $O_2$ than vice versa) then $O_1$ is to be preferred to $O_2$. Furthermore, class subsumptions, which can be inferred from both $O_1$ or $O_2$, should be of no consequence to determine which repaired ontology is preferable. 
That is, 
%
%
whenever $\Inf(O_1) \subseteq \Inf(O'_1)$, $\Inf(O_2) \subseteq \Inf(O'_2)$ and $\Inf(O'_1) \setminus \Inf(O_1) = \Inf(O'_2) \setminus \Inf(O_2)$ it should hold that the quality of $O_1$ with respect to $O_2$ is the same as the quality of $O'_1$ with respect to $O'_2$. 
%
Thus, we 
define the following measure to compare the inferable information content of two ontologies. 
\begin{definition}
    Let $O_1$ and $O_2$ be two consistent ontologies. If $\Inf(O_1) \not = \Inf(O_2)$, we define the \emph{inferable information content} $\qual(O_1, O_2)$ of $O_1$ w.r.t.\ $O_2$ as $\qual(O_1, O_2) = $
	\[
    \frac{\mathbf{card}(\Inf(O_1) \setminus \Inf(O_2))}{\mathbf{card}(\Inf(O_1) \setminus \Inf(O_2))  + \mathbf{card}(\Inf(O_2) \setminus \Inf(O_1))} \enspace ;
	\]
	if instead $\Inf(O_1) = \Inf(O_2)$, we set $\qual(O_1, O_2) =0.5$. 
\end{definition}
It is readily seen that this definition satisfies the two conditions mentioned above. Furthermore, 
the following properties hold: 
\begin{enumerate}
	\item $\qual(O_1, O_2) \in [0,1]$; 
	\item $\qual(O_1, O_2) = 1 - \qual(O_2, O_1)$; 
	\item $\qual(O_1, O_2) = 0.5$ if and only if $\mathbf{card}(\Inf(O_1)) = \mathbf{card}(\Inf(O_2))$; 
	\item $\qual(O_1, O_2) = 1$ if and only if $\Inf(O_2) \subset \Inf(O_1)$; 
	\item $\qual(O_1, O_2) > 0.5$ if and only if $\mathbf{card}(\Inf(O_1) \setminus \Inf(O_2)) > \mathbf{card}(\Inf(O_2) \setminus \Inf(O_1))$.
\end{enumerate}
Although this is by no means the only possible measure for comparing two ontologies~\cite{tartir2005ontoqa,Alani:2006:ROA:2127045.2127047,vrandevcic2007design,vrandevcic2009ontology}, these properties suggest that our definition captures a notion of ``quality'' that is meaningful for our intended application: in particular, if for two proposed repairs $O_1, O_2$ of an inconsistent ontology $O$ we have $\qual(O_1, O_2) > 0.5$, then there are more class subsumptions which can be inferred in $O_1$ but not in $O_2$ than vice versa, and hence---all other things being equal---$O_1$ is a better repair of $O$ than $O_2$.

One possible criticism of our definition of $\qual(O_1, O_2)$ is that its value depends only on $\Inf(O_1)$ and $\Inf(O_2)$: if $O_1$ and $O_2$ differ only w.r.t.\ subsumptions between complex concepts, then 
$\qual(O_1, O_2) = 0.5$ (even though the implications of $O_1$ might still be considerably richer than those of $O_2$). On the other hand, focusing on atomic subsumptions makes also conceptual sense, as these are the ones that our inconsistent ontology---as well as the proposed repairs---discuss about. It is, in any case, certainly true that our measure is fairly \emph{coarse}: if $\qual(O_1, O_2)$ is significantly greater than $0.5$ there are good grounds to claim that $O_1$ is a better repair of $O$ than $O_2$ is, but it may easily be
that repair candidates between which our measure cannot discriminate are nonetheless of different quality.

To empirically test whether weakening axioms is a better approach to ontology repair than removing them, we tested our approach on ten ontologies from BioPortal~\cite{bioportal}, expressed in \ALC (see Table~\ref{table:ontologies}). On average the ontologies have 105 logical axioms and 90 classes. We compared the performance of RepairOntologyWeaken (Algorithm~\ref{algo:repair_weaken}) with the one of the non weakening-based RepairOntologyRemove (Algorithm~\ref{algo:repair_remove}) by first making the ontologies inconsistent through the addition of random axioms, then attempting to repair them through the two algorithms (using the original ontology as the reference), and then computing \qual.\footnote{Limited to the `mis' variant of \text{FindBadAxiom}, the prototype implementation of \text{RepairOntologyWeaken}, \text{RepairOntologyRemove}, and the \qual measure is provided as supplemental material.} 
This procedure has the following rationale: one may think that the axioms added constitute some new claims made concerning the relationships between the classes of the ontology, which however unfortunately made it inconsistent. It is thus desirable to fix this inconsistency while preserving as much as possible of the informational content of these axioms and of the other axioms in the ontology.

The procedure was repeated one hundred times per ontology, selecting the axioms to weaken or remove by sampling minimally inconsistent sets, and one further hundred times selecting the axioms to remove or weaken completely randomly. 
We tested the significance of our results through both Wilcoxon signed-rank tests and T-tests, applying the Holm-Bonferroni correction for multiple comparison, with a p-value threshold of $0.05$. 

Figure~\ref{fig:evaluation} and Table~\ref{table:results} summarise the results of our experiments. When choosing the axioms to weaken or remove through sampling minimally inconsistent sets, the means (in the case of the T-test) and medians (in the case of the Wilcoxon test) of the \qual 
for \text{RepairOntologyWeaken} against \text{RepairOntologyRemove} were all significantly greater than $0.5$ for all ontologies. This confirms that our repair-by-weakening technique is able to preserve more of the informational content of axioms than repair-by-removal techniques. When selecting the axioms to repair randomly, on the other hand, this was not always the case, as shown in Table \ref{table:results}. 
This illustrates how our weakening-approach on ontology repair reliably constitutes an improvement over removal-based ontology repair only when problematic axioms can be reliably pinpointed. 

Figure~\ref{fig:evaluation} highlights the effect of choosing the axioms to repair or remove randomly rather than through sampling inconsistent sets. While the difference is not statistically significant for all ontologies,\footnote{For instance, w.r.t.\ the Wilcoxon test it is statistically significant for \emph{bctt}, \emph{elig}, \emph{ogr}, \emph{pe} and \emph{taxrank}, but not for the other five ontologies.} we observe that the quality of the repair compared to the corresponding removal is always improved by choosing the axioms to repair via sampling. The natural next step in this line of investigation would consist in evaluating the effect of varying the number of minimally inconsistent sets sampled by FindBadAxiom, which for these experiments was set to one tenth of the ontology size.

To summarise, the main conclusion of our experiments is that, when problematic axioms can be reliably identified, our approach is better able to preserve the informational content of inconsistent ontologies than the corresponding repair-by-removal method.

\section{Related Work}
Refinements operators were also discussed in the context of inductive logic programming~\cite{laag98}, and formalised in description logics for concept learning~\cite{LeHi10}. The refinement operators used by our weakening approach were introduced in~\cite{amai2016}, and further analysed in~\cite{dl2017-coherence} in relation to incoherence detection. They were not previously applied to ontology repair. 

The problem of identifying and repairing inconsistencies in ontologies has received much attention in recent years. Our approach differs from many other works in the area, see for instance~\cite{ScCo03,kalyanpur2005debugging,kalyanpur2006repairing,BaPS07,haase2007analysis}, in that---rather than removing the problematic axioms altogether---we attempt to repair the ontology by replacing them with weakened rewritings. On the one hand, our method requires the choice of a (consistent) \emph{reference ontology} with respect to which one can compute the weakenings; on the other hand, it allows us to perform a softer, more fine-grained form of ontology repair.

A different approach for repairing ontologies through weakening was discussed
in~\cite{lam2008fine}. Our own approach is, however, quite different from it: while the repair algorithm of~\cite{lam2008fine} operates by pinpointing (and subsequently removing) the subcomponents of the axioms responsible for the contradiction, ours is based on a refinement operator, which combines both semantic (via the cover operators) and syntactic (via the compositional definitions of generalisations and specialisations of complex formulas) information in order to identify candidates for the replacement of the offending axiom(s). In particular, this implies---using the terminology of~\cite{ji2014measuring}---that our repair algorithm, in contrast to~\cite{lam2008fine}, is `black box' in that it treats the reasoner  as an oracle, and can thus be more easily combined with different choices of reasoner (or, with slightly more effort, applied to different logics). 

Another influential approach to ontology repair is discussed in~\cite{qi2006revision} and in~\cite{qi2006knowledge}. That approach, like ours, attempts to weaken problematic axioms; but it does so by adding \emph{exceptions} to value restrictions $\forall R.C(a)$,\footnote{Another difference is that we are also interested in repairing TBoxes, whereas the approach of~\cite{qi2006knowledge} operates only over ABoxes.} rather than by means of a more general-purpose transformation. 

We leave to future work the evaluation of our approach in comparison to other state-of-the-art ontology repair frameworks. As already stated, this is not an entirely well-posed problem; but if, as in this work, we accept that a suggested repair $O_1$ is preferable to another suggested repair $O_2$ whenever $\mathbf{card}(\Inf(O_1) \setminus \Inf(O_2)) > \mathbf{card}(\Inf(O_2) \setminus \Inf(O_1))$ then the question becomes amenable to analysis. Possibly, complementary metrics for further evaluations can be chosen from~\cite{Alani:2006:ROA:2127045.2127047}. Experiments involving user evaluation could be also considered in this context.


\section{Conclusions} 

We have proposed a new strategy for repairing ontologies based on the idea of weakening terminological and assertional axioms. Axiom weakening is a way to improve the balance between regaining consistency and keeping as much information from the original ontology as possible. 

We have investigated the theoretical properties of the refinement operators that are required in the definition of axiom weakening and analysed the computational complexity of employing them. Furthermore, the empirical evaluation shows that our weakening-based approach to repairing ontologies performs significantly better, in terms of preservation of information, than the removal-based approach.  

Future work will concentrate on the following two directions. Firstly, we plan to extend our evaluation to further corpora of ontologies and measures of information, including in particular measures to reflect the syntactic complexity of repaired ontologies and measures that reflect the preservation of entailments of competency questions. Secondly, we plan to extend the presented approach to axiom weakening to more expressive DL languages, including $\mathcal{SROIQ}$  underlying OWL 2 DL  \cite{SROIQ} and to full first-order logic, for which debugging is a particularly challenging problem \cite{dolcecon11}.  We expect that, for more complex languages, the weakening-based strategy will likewise significantly improve on the removal-based strategy, and indeed be even more appropriate by exploiting the higher syntactic complexity. 
\nb{our list of reference is dangerously long. It does not leave space to answer the reviewers comments in the camera ready, for which we can use the 8th page. In theory we are not allowed to remove a reference from the submitted version. We need to add PRIMA: OK: we decided to leave out PRIMA as it is unpublished and not accessible; regarding citations, I don't think it is a problem to change the bibliography after acceptance, we can always squeeze back into 8 pages as we like.
RPN: I don't think we should worry too much about how to address reviewers' comments at this point. We should rather submit the strongest paper we can! (and there are several places where we can gain a bit of space when it becomes necessary)}

\bibliographystyle{aaai}
\bibliography{aaai18}


\section*{Appendix}

\begin{proposition}\label{prop:bounded-sub}
Let \Tmc be an \ALC TBox. We have \[ \mathbf{card}(\sub(\mathcal{T})) \leq \sum_{C \sqsubseteq D} (\mid C \mid + \mid D \mid) + 2 \enspace .\]
\end{proposition}
This is a consequence of Lemma~\ref{lem:bounded-sub-concept}.
\begin{lemma}\label{lem:bounded-sub-concept}
Let $C$ be an \ALC concept. We have $\mathbf{card}(\sub(C)) \leq \mid C \mid$.
\end{lemma}
\begin{proof}
If $C \in N_C$, or $C = \top$, or $C = \bot$, then the result holds. Hence, it holds for all concepts of size~$1$.

If $C = \lnot A$, with $A \in N_C$, then the result holds. (We have $\mid C \mid = \mathbf{card}(\sub(C)) = 2$.) Hence, it holds for all concepts of size~$2$.

Let $k \geq 1$, and 
suppose for induction that
for every concept $C$, if $\mid C \mid \leq k$, then $\mathbf{card}(\sub(C)) \leq \mid C \mid$.

Let $C$ be a concept such that $\mid C \mid = k + 1$.

Case $C = \lnot A$. The result holds.

Case $C = E \sqcap F$. Clearly, $\mid E \mid \leq k$ and $\mid F \mid \leq k$. Thus, by IH, $\mathbf{card}(\sub(E)) \leq \mid E \mid$ and $\mathbf{card}(\sub(F)) \leq \mid F \mid$. Hence, $\mathbf{card}(\sub(C))\leq \mid E \mid + \mid F \mid + 1 = \mid C \mid$.

Case $C = E \sqcap F$. Idem.

Case $C = \exists R. D$. Clearly, $\mid D \mid = k$. Thus, by IH, $\mathbf{card} \sub(D)) \leq \mid D \mid$. Hence, $\mathbf{card}(\sub(C))\leq \mid D \mid + 1 = \mid C \mid$.

Case $C = \forall R. D$. Idem.
\end{proof}

\begin{lemma}\label{lem:gamma-is-generalisation}
For every \ALC concept $C$, and for every $C' \in \gamma(C)$, we have $C \sqsubseteq C'$.
\end{lemma}
\begin{proof}
  By induction on the complexity of the concept $C$.
  
  Case $C \in N_C$, or $C = \top$, or $C = \bot$. $\gamma(C) = \mathsf{UpCov}(C)$. By definition of $\mathsf{UpCov}$, for every $C' \in \mathsf{UpCov}(C)$ we have $C \sqsubseteq C'$.

  \medskip
  
  Case $C = \lnot A$, with $A \in N_C$. 
  Let $C' \in \gamma(C)$. Either $C' \in \mathsf{UpCov}(C)$, for which the result is immediate, or there is $E \equiv \lnot C'$ such that $E \in \mathsf{DownCov}(A)$. By definition of $\mathsf{DownCov}$, we have $E \sqsubseteq A$, and thus $\lnot C' \sqsubseteq \lnot C$. Hence, $C \sqsubseteq C'$.

If $C' \in \gamma(C) \cap \mathsf{UpCov}(C)$, it is clear from the definition of $\mathsf{UpCov}$ that $C \sqsubseteq C'$. We'll ignore this subcase in the cases.

  \medskip
  
  Now, suppose for induction that the property holds for the concepts $E$ and $F$. Let $R$ be an arbitrary role in $N_R$.
  
  Case $C = E \sqcap F$. Subcase $C'$ is of the form $E' \sqcap F$, with $E' \in \gamma(E)$. By IH, we have $E \sqsubseteq E'$. Indeed, $C \sqsubseteq C'$. Subcase $C'$ is of the form $E \sqcap F'$, with $F' \in \gamma(F)$. By IH, we have $F \sqsubseteq F'$. Indeed, $C \sqsubseteq C'$.

  Case $C = E \sqcup F$. Subcase $C'$ is of the form $E' \sqcup F$, with $E' \in \gamma(E)$. By IH, we have $E \sqsubseteq E'$. Indeed, $C \sqsubseteq C'$. Subcase $C'$ is of the form $E \sqcup F'$, with $F' \in \gamma(F)$. By IH, we have $F \sqsubseteq F'$. Indeed, $C \sqsubseteq C'$.

  Case $C = \forall R. E$. Let $C' \in \gamma(C)$. Thus, $C'$ is of the form $\forall R. E'$. By IH, we have $E \sqsubseteq E'$. Indeed, $C \sqsubseteq C'$.

  Case $C = \exists R. E$. Let $C' \in \gamma(C)$. Thus, $C'$ is of the form $\exists R. E'$. By IH, we have $E \sqsubseteq E'$. Indeed, $C \sqsubseteq C'$.
\end{proof}

\begin{lemma}\label{lemma:gamma-finite}
For every finite TBox $\Tmc$ and for every concept $C$, the set of concepts $\gamma_{\Tmc}(C)$ is finite.
\end{lemma}
\begin{proof}
The set $\sub(\Tmc)$ is finite; In fact linear in the size of $\Tmc$, Lemma~\ref{prop:bounded-sub}.

For every concept $C$, we have $\mathsf{UpCov}_\Tmc(C) \subseteq \sub(\Tmc)$. So $\mathsf{UpCov}_\Tmc(C)$ is linearly bounded by $|\Tmc|$.

For every concept $C$, we have $\mathsf{DownCov}_\Tmc(C) \subseteq \sub(\Tmc)$. So $\mathsf{DownCov}_\Tmc(C)$ is linearly bounded by $|\Tmc|$.

For every concept $C$, the size $|\mathsf{nnf}(C)|$ is linearly bounded by $|C|$.

Finally, all recursive calls of $\gamma_{\Tmc}$ are done on concepts of strictly decreasing size.
\end{proof}
We can do better, and bound the size of $\gamma_\Tmc(C)$ linearly.
\begin{lemma}\label{lemma:gamma-bounded}
For every finite TBox $\Tmc$ and for every concept $C$, we have $\mathbf{card}(\gamma_{\Tmc}(C)) \leq (|\Tmc| + 2) \cdot |C|$.
\end{lemma}
\begin{proof}
The proof is done by induction on the complexity of the concept $C$.

When $C \in N_C$, or $C = \top$, or $C = \bot$, $\gamma_\Tmc(C) = \mathsf{UpCov}_\Tmc(C)$. Moreover $\mathbf{card}(\mathsf{UpCov}_\Tmc(C)) \leq \mathbf{card}(\sub(\Tmc))$. Proposition~\ref{prop:bounded-sub} ensures that the result holds.

When $C = \lnot A$, $\mathbf{card}(\gamma_\Tmc(C)) = \mathbf{card}(\mathsf{DownCov}_\Tmc(A)) + \mathbf{card}(\mathsf{UpCov}_\Tmc(C))$.
Since $\mathbf{card}(\mathsf{DownCov}_\Tmc(A)) \leq \mathbf{card}(\sub(\Tmc))$, Proposition~\ref{prop:bounded-sub} ensures the result again.

Now suppose for induction that the result holds for the concepts $D$ and $E$.

When $C = D \sqcap E$ or $C = D \sqcup E$, $\mathbf{card}(\gamma_\Tmc(C)) \leq \mathbf{card}(\gamma_\Tmc(D)) + \mathbf{card}(\gamma_\Tmc(E)) + \mathbf{card}(\mathsf{UpCov}_\Tmc(C))$.
By I.H., $\mathbf{card}(\gamma_\Tmc(C)) \leq (|D| \cdot (|\Tmc| + 2)) + (|E| \cdot (|\Tmc| + 2)) + (|\Tmc| + 2)$, which yields the result.

When $C = \exists r. D$ or $C = \forall r. D$, 
$\mathbf{card}(\gamma_\Tmc(C)) \leq \mathbf{card}(\gamma_\Tmc(D)) + \mathbf{card}(\mathsf{UpCov}_\Tmc(C))$.
By I.H., $\mathbf{card}(\gamma_\Tmc(C)) \leq (|D| \cdot (|\Tmc| + 2)) + (|\Tmc| + 2)$, which yields the result.
\end{proof}

\paragraph{Proof of Generalisability}
We prove in detail the generalisability property: 
\begin{center}
If $C, D \in \sub(\Tmc)$ and $C \sqsubseteq_{\Tmc} D$ then $D \in \gamma_{\Tmc}^*(C)$ \enspace .
\end{center}
We will do so by proving a similar statement about the upwards cover operator: 
\begin{center}
If $C, D \in \sub(\Tmc)$ and $C \sqsubseteq_{\Tmc} D$ then $D \in \mathsf{UpCov}_{\Tmc}^*(C)$ \enspace , 
\end{center}
where 
  \begin{itemize}
  \item $\mathsf{UpCov}_{\Tmc}^0(C) = \{C\}$;
  \item $\mathsf{UpCov}_{\Tmc}^{j+1}(C) = \mathsf{UpCov}_{\Tmc}^{j}(C) \cup \bigcup_{C' \in \mathsf{UpCov}_{\Tmc}^j(C)} \mathsf{UpCov}_{\Tmc}(C')$ \enspace, $j \geq 0$;
  \item $\mathsf{UpCov}_{\Tmc}^*(C) = \bigcup_{i \geq 0} \mathsf{UpCov}_{\Tmc}^i(C)$.
  \end{itemize}
First, we will prove the latter statement; then, we will use it (as well as the definition of $\gamma_{\Tmc}$) to prove the former. 

\begin{lemma}
Let $C, D \in \sub(\Tmc)$, $C \sqsubseteq_\Tmc D$ and $D \not \in \mathsf{UpCov}_{\Tmc}(C)$ then there exists some $C' \in \mathsf{UpCov}_{\Tmc}(C)$ such that $C \sqsubset_\Tmc C' \sqsubset_\Tmc D$. 
\label{lemma:upcov_from_below}
\end{lemma}
\begin{proof}
We prove this by contradiction. Assume that $C, D \in \sub(\Tmc)$, $C \sqsubseteq_\Tmc D$, $D \not \in \mathsf{UpCov}_{\Tmc}(C)$, and that for all $E \in \sub(\Tmc)$ with $C \sqsubset_\Tmc E \sqsubset D$ it is the case that $E \not \in \mathsf{UpCov}_\Tmc$. Then we will prove, by induction on $n \in \mathbb N$, that for any integer $n$ there exists a chain $C \sqsubset_\Tmc D_n \sqsubset_\Tmc D_{n-1} \sqsubset_\Tmc \ldots \sqsubset_\Tmc D_1 \sqsubset_\Tmc D$ of concepts $D_1 \ldots D_n \in \sub(\Tmc)$.

This will imply a contradiction, because $\sub(\Tmc)$ is finite\footnote{It would suffice to know that the length of subsumption chains $A_1 \sqsubset_\Tmc A_2 \sqsubset_\Tmc \ldots \sqsubset_\Tmc A_n$ of elements in $\sub(\Tmc)$ is finite and bounded.} and because these $D_1 \ldots D_n$ will be pairwise distinct (since $D_i \sqsubset_\Tmc D_j$ whenever $i > j$).

Let us proceed with the induction. 
\begin{itemize}
\item \textbf{Base case: } Let $n =1$. Since $C \sqsubseteq_\Tmc D$ but $D \not \in \mathsf{UpCov}_\Tmc(C)$, by definition of the upward cover set there exists some $D_1 \in \sub(\Tmc)$ such that $C \sqsubset_\Tmc D_1 \sqsubset_\Tmc D$, as required. 
\item \textbf{Inductive case: } Suppose that the statement holds for $n$, i.e., there exist $D_1 \ldots D_n \in \sub(\Tmc)$ such that $C \sqsubset_\Tmc D_n \sqsubset_\Tmc D_{n-1} \sqsubset_\Tmc \ldots \sqsubset_\Tmc D_1 \sqsubset_\Tmc D$. 

Now by hypothesis, since $C \sqsubset_T D_n \sqsubset_\Tmc D$, it must be the case that $D_n \not \in \mathsf{UpCov}_\Tmc(C)$. But then, again by definition of upward cover set, there must exist some $D_{n+1} \in \sub(\Tmc)$ with $C \sqsubset_\Tmc D_{n+1} \sqsubset_\Tmc D_n$, as required.
\end{itemize}
\end{proof}

Using this lemma, we can now prove the generalisability property of the upward cover: 
\begin{theorem}
Let $C, D \in \sub(\Tmc)$ be such that $C \sqsubseteq_\Tmc D$. Then $D \in \mathsf{UpCov}^*_\Tmc(C)$, that is, there exist $n \in \mathbb N$ and $C_1, \ldots C_n \in \sub(\Tmc)$ such that 
\begin{enumerate}
\item $C_1 \in \mathsf{UpCov}_\Tmc(C)$; 
\item For all $i < n$, $C_{i+1} \in \mathsf{UpCov}_\Tmc(C_i)$; 
\item $C_n = D$. 
\end{enumerate}
\label{theo:gene_upcov}
\end{theorem}
\begin{proof}
Assume that  $C, D \in \sub(\Tmc)$, $C \sqsubseteq_\Tmc D$ but $D \not \in \mathsf{UpCov}_\Tmc^*(C)$. We prove, by induction over $n \in \mathbb N$, that for any integer $n$ there exists a chain $C \sqsubset_\Tmc C_1 \sqsubset_\Tmc C_2 \sqsubset_\Tmc \ldots \sqsubset_\Tmc C_n \sqsubset_\Tmc D$ of pairwise distinct concepts $C_1 \ldots C_n \in \sub(\Tmc)$ such that $C_1 \in \mathsf{UpCov}_\Tmc(C)$ and $C_{i+1} \in \mathsf{UpCov}_\Tmc(C_i)$ for $i < n$. As in Lemma \ref{lemma:upcov_from_below}, this is a contradiction, because $\sub(\Tmc)$ is finite.
\begin{itemize}
\item \textbf{Base case:} Let $n = 1$. Since $D \not \in \mathsf{UpCov}^*_\Tmc(C)$, in particular it is the case that $D \not \in \mathsf{UpCov}_\Tmc(C)$.\footnote{Indeed, otherwise the trivial chain $C_1 = D$ would prove that $D \in \mathsf{UpCov}_\Tmc(C)$.} Then by Lemma \ref{lemma:upcov_from_below}, there exists some $C_1 \in \mathsf{UpCov}_\Tmc(C)$ such that $C \sqsubset_\Tmc C_1 \sqsubset_\Tmc D$, as required.
\item \textbf{Inductive case:} Suppose that the statement holds for $n$, i.e., there exist $C_1 \ldots C_n \in \sub(\Tmc)$ such that $C \sqsubset_\Tmc C_1 \ldots \sqsubset_\Tmc C_n \sqsubset_\Tmc D$, $C_1 \in \mathsf{UpCov}_\Tmc(C)$, and $C_{i+1} \in \mathsf{UpCov}_\Tmc(C_i)$ for all $i < n$. 

Then it must be the case that $D \not \in \mathsf{UpCov}_\Tmc(C_n)$, since otherwise it would be true that $D \in \mathsf{UpCov}_\Tmc^*(C)$ by means of the chain $C_1 C_2 \ldots C_n D$. But then, again by Lemma \ref{lemma:upcov_from_below}, it is also true that there exists some $C_{n+1} \in \mathsf{UpCov}_\Tmc(C_n)$ such that $C_n \sqsubset_\Tmc C_{n+1} \sqsubset_\Tmc D$, as required. 
\end{itemize}
\end{proof}
A straightforward consequence of this result is that our generalisation operator also satisfies the generalisability property: 
\begin{corollary}\label{cor:generalisability}
Let $C, D \in \sub(\Tmc)$ be such that $C \sqsubseteq_\Tmc D$. Then $D \in \gamma^*_\Tmc(C)$.
\end{corollary}
\begin{proof}
By Theorem \ref{theo:gene_upcov} we have at once that $D \in \mathsf{UpCov}_\Tmc^*(C)$, i.e., there exist concepts $C_1 \ldots C_n$ such that $C_1 \in \mathsf{UpCov}_\Tmc(C)$, $C_{i+1} \in \mathsf{UpCov}_\Tmc(C_i)$ for all $i < n$, and $C_n = D$. But by 
Lemma~\ref{lem:gen}.\ref{item:relevant-completeness},
$\mathsf{UpCov}_\Tmc(E) \subseteq \gamma_\Tmc(E)$ for every concept $E$; thus, the same chain $C \rightarrow^\gamma C_1 \rightarrow^\gamma C_2 \ldots \rightarrow^\gamma C_n=D$ also demonstrates that $D \in \gamma^*_\Tmc(C)$. 
\end{proof}


\end{document}